\documentclass{article}

\usepackage[accepted]{icml2026}

\usepackage[utf8]{inputenc}
\usepackage[T1]{fontenc}
\usepackage{microtype}

\usepackage{amsmath}
\usepackage{amssymb}
\usepackage{amsfonts}
\usepackage{mathtools}
\usepackage{amsthm}
\usepackage{nicefrac}

\usepackage{graphicx}
\usepackage{subcaption}
\usepackage{booktabs}
\usepackage{adjustbox}
\usepackage{rotating}
\usepackage{pdflscape}
\usepackage{makecell}
\usepackage{threeparttable}

 \usepackage{algorithm}
 
\usepackage{tikz}
\usetikzlibrary{decorations.pathreplacing, calc, positioning, arrows.meta}
\usepackage{pgfplots}
\pgfplotsset{compat=1.18}

\usepackage{xcolor}
\usepackage{fontawesome5}

\usepackage{hyperref}
\usepackage[capitalize,noabbrev]{cleveref}

\usepackage{placeins}
\usepackage[textsize=tiny]{todonotes}

\theoremstyle{plain}
\newtheorem{theorem}{Theorem}[section]

\newtheorem{lemma}[theorem]{Lemma}

\theoremstyle{definition}
\newtheorem{definition}[theorem]{Definition}

\theoremstyle{remark}

\newcommand{\R}{\mathbb{R}}
\newcommand{\Rp}{\mathbb{R}_{>0}}

\definecolor{accept}{HTML}{1B7F3B}
\definecolor{reject}{HTML}{B00020}
\newcommand{\cmark}{\textcolor{accept}{\ensuremath{\checkmark}}}
\newcommand{\xmark}{\textcolor{reject}{\ensuremath{\boldsymbol{\times}}}}
\newcommand{\hlclause}[1]{\colorbox{red!15}{#1}}

\newcommand{\examplebrace}[2]{%
  \noindent
  \begin{tikzpicture}
    \node[inner sep=0pt, anchor=north west] (content) {%
      \begin{minipage}[t]{0.68\linewidth}
        #2
      \end{minipage}
    };
    \draw[decorate, decoration={brace, amplitude=4pt, raise=2pt}, line width=0.4pt]
      ($(content.north east)+(2pt,0)$) -- ($(content.south east)+(2pt,0)$)
      node[midway, xshift=12pt, anchor=west, align=left, font=\itshape\footnotesize] {#1};
  \end{tikzpicture}
  \par\vspace{0.8em}
}

\icmltitlerunning{LemmaBench: A Live, Research-Level Benchmark to Evaluate LLM Capabilities in Mathematics}

\begin{document}

\twocolumn[
  \icmltitle{LemmaBench: \\ A Live, Research-Level Benchmark to Evaluate LLM Capabilities in Mathematics}

  \icmlsetsymbol{equal}{*}

  \begin{icmlauthorlist}
    \icmlauthor{Antoine Peyronnet}{equal,yyy,ppp}
    \icmlauthor{Fabian Gloeckle}{equal,ppp}
    \icmlauthor{Amaury Hayat}{equal,ppp,ooo}
  \end{icmlauthorlist}

  \icmlaffiliation{yyy}{Ecole Normale Superieure de Rennes}
  \icmlaffiliation{ppp}{Ecole des Ponts Paris}
  \icmlaffiliation{ooo}{Korean Institute for Advanced Study}

  \icmlcorrespondingauthor{Antoine Peyronnet}{antoine.peyronnet@ens-rennes.fr}

  \icmlkeywords{Machine Learning, AI for Mathematics, Benchmarking, Theorem Proving}

  \vskip 0.3in
]

\printAffiliationsAndNotice{}

\begin{abstract}
We present a new approach for benchmarking Large Language Model (LLM) capabilities on research-level mathematics. Existing benchmarks largely rely on static, hand-curated sets of contest or textbook-style problems as proxies for mathematical research. Instead, we establish an updatable benchmark evaluating models directly on the latest research results in mathematics. This consists of an automatic pipeline that extracts lemmas from the arXiv and rewrites them into self-contained statements by making all assumptions and required definitions explicit. It results in a benchmark that can be updated regularly with new problems taken directly from human mathematical research, while previous instances can be used for training without compromising future evaluations. We benchmark current state-of-the-art LLMs, which went from 10-15$\%$ accuracy to 40.8\% accuracy in theorem proving (pass@1) depending on the model, showing that there is both a very fast evolution of LLMs capabilities with respect to human research-level mathematics and still a large margin of progression for general public LLMs to reach human-level proving capabilities in a research context.
\end{abstract}

\section{Introduction}

Deep learning and machine learning methods have enjoyed tremendous success over the past decade,  culminating in Nobel Prizes in physics and chemistry in 2024. It has become increasingly clear that AI methods are changing the practice of science, sometimes to the surprise of scientists and researchers in these disciplines. While machine learning methods are already useful in mathematics \citep{davies2021advancing,charton2024patternboost,alfarano2024global}, theorem proving remains a major open challenge. 

The exceptional progress made by Large Language Models (LLMs) 
in recent months and years offers a glimpse into a future where models could be capable of proving theorems autonomously at a human level. The performances of LLMs in this area have exploded over the past two years, mainly thanks to two elements: first, the use of end-to-end reinforcement learning from verifiable rewards (RLVR), which has led to the advent of  so-called \emph{reasoning models} \cite{guo2025deepseek,rastogi2025magistral}; second, a major effort to gather and generate large mathematical data corpora, particularly for Olympiad-style exercises whether in natural language \cite{gao2024omni,numina_math_datasets,fang2025mathodyssey} or in formal languages \cite{ahm2025leanuniverse,li2024numinamath,ying2024lean}. All of this has led to several achievements: announcements of mathematical models capable of winning a gold medal at the International Mathematical Olympiads 2025 \cite{huang2025gemini,harmonic2025aristotle,castelvecchi2025ai} the arrival of reasoning models that prove to be highly efficient in mathematical symbolic computations, capable of performing in a few minutes the equivalent of several hours of human work \cite{diez2025mathematical}; and the resolution of several open problems with the help of LLM-based automated theorem provers \cite{chen2026fel,sothanaphan2026resolution,liu2026numina}, specialized agent tools \cite{dixit2026project} or even generic LLMs \cite{schmitt2025extremal,feng2026semi,jang2025point}.
While such strong performance on research-level mathematics has been claimed \cite{frontiermath}, the true capabilities are still debated.

While current models perform very well in calculations and in undergrad curriculum mathematics, two questions arise: 
\begin{itemize}
\item \textbf{Data contamination}: how much of the performances of such models are tainted by possible contamination of the training data?  
\item \textbf{Generalization and performance in research-level mathematics}: how can these models, often post-trained on curated or competition-style math datasets, perform on specific areas of research-level mathematics where, by definition, data is scarce? 
\end{itemize}

This paper proposes another approach: relying on mathematicians' production to establish a research-level benchmark that can be updated as often as needed.
Our pipeline, inspired by LiveCodeBench \citep{jainlivecodebench} in the domain of competitive coding, works as follows: it extracts lemmas from the latest research preprints on the arXiv and automatically gathers the required assumptions and definitions from elsewhere in the paper, thus making them self-contained statements that can be used as stand-alone benchmark elements.
The advantages of such an automatically created benchmark are as follows:
\begin{itemize}
\item[\textbf{(i)}] It can be updated regularly at low cost (up to weekly if needed), allowing contamination-free evaluation at all times; Human evaluation is only needed to assess the confidence score of the LLM as a judge and can be updated much less frequently (essentially when the judge is changed).
\item[\textbf{(ii)}] The same pipeline can be run on historical data, allowing the creation of large high-quality training datasets;
\item[\textbf{(iii)}] Domain composition can be steered freely, and domain-specific benchmarks can be created based on the categorization of research preprints on the arXiv;
\item[\textbf{(iv)}] The data is perfectly in-distribution for actual research-level lemma proving by definition of the data source.
\end{itemize}

Our \textbf{contributions} can be summarized as follows:
\begin{itemize}
    \item We describe this \textbf{pipeline} for generating a live high-quality benchmark on lemma-proving in research-level mathematics from preprints on the arXiv, taking care to ensure \emph{self-containedness} of the resulting statements which has been identified as a key limitation of data from the arXiv \cite{patel2023new,alexander2026semantic} 
    (Section~\ref{sec:extract}).
    \item We \textbf{evaluate the quality} of the resulting self-contained lemma benchmark with manual inspection, indicating high precision in producing exact and correct mathematical statements (Section~\ref{ssec: self-containement test}).
    \item Using the benchmark, we evaluate the \textbf{lemma-proving accuracy} of current state-of-the-art LLMs. Two evaluations are considered: the proofs provided are first evaluated with an LLM as a judge. A subset of the proofs are then judged by human expert mathematicians so as to provide a confidence score to the LLM as a judge evaluation.
    (Section~\ref{sec:gen-judge}). The main results are provided in Tables \ref{tab:results-compact} and \ref{tab:results-compact-2026}. 
\end{itemize}

Instead of evaluating the theoretical upper bound of LLM performance in theorem proving, this paper asks a more practical question: how much of today's research-level mathematical theorem proving can already be performed using widely available, general-purpose LLMs, the tools most mathematicians are realistically expected to use in the near future—rather than specialized or advanced systems available only to a small number of researchers. We find that the capabilities of these general-purpose models have progressed rapidly over the past nine months, while a substantial gap to research-level proficiency remains.

\begin{figure}[h!]
    \centering
    \includegraphics[width=0.8\linewidth]{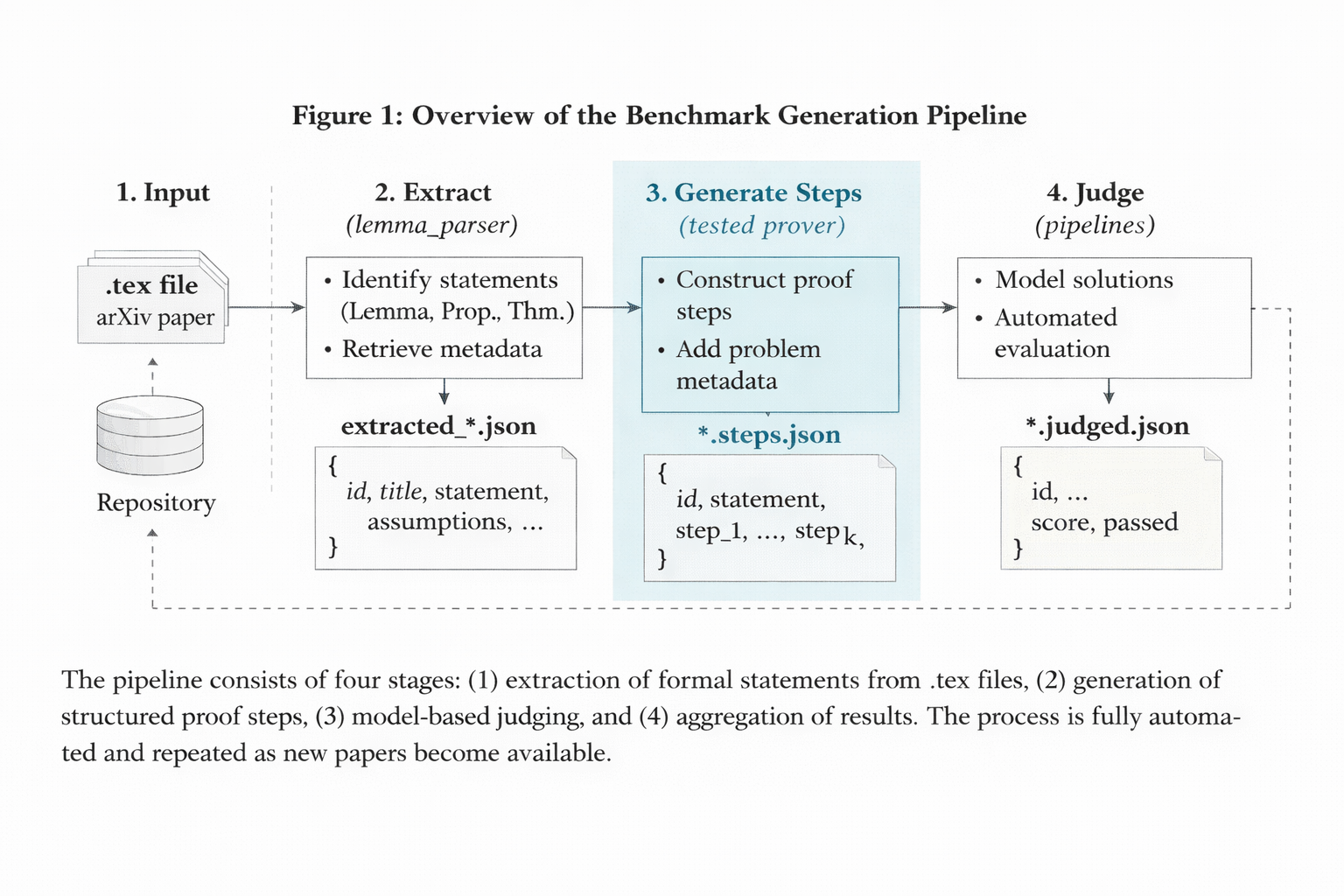}
    \caption{\textbf{Pipeline at a glance}: the pipeline retrieves papers from the arXiv, extract statements and make them self-contained. An external prover can then try to prove them step by step and each proof is inspected by a judge. The final score is the proportion of valid proofs. We used GPT-5, Gemini 2.5 Pro, Gemini 3 Pro, Claude Opus 4.5, and DeepSeek-R1 as a baseline for provers.}
    \label{fig:placeholder}
\end{figure}

\section{Related Work}
Early mathematical reasoning benchmarks such as GSM8K \citep{cobbe2021training},
MathQA \citep{amini2019mathqa} and MATH \citep{hendrycks2021measuring} target school- to
undergraduate-level problems and are now largely saturated and exposed to
contamination \citep{xu2024benchmarking}. A second line of work relies on expert-curated,
often closed problem sets to push difficulty and limit leakage: FrontierMath
\citep{frontiermath} keeps its problems closed-source, Soohak \citep{son2026soohak}
compiles over 400 problems from 64 mathematicians, First Proof
\citep{abouzaid2026first} releases ten expert questions for a single one-week
evaluation, and the Leipzig Tier-4 set \citep{leipzig_benchmarks_2026} offers 102 open
research-level problems. These approaches are valuable but costly to scale and
update.

Closest to ours is the emerging family of \emph{live}, arXiv-derived benchmarks.
RealMath \citep{zhang2025realmath} and MathArena/ArxivMath \citep{arxivmath2026} extract
problems from recent papers but restrict themselves to statements with
automatically verifiable answers; EternalMath \citep{ma2026eternalmath} similarly targets
code-verifiable questions. LemmaBench builds on this direction but differs in a
key respect: rather than limiting the benchmark to verifiable final answers, we
evaluate \emph{natural-language proofs} of research lemmas. This is harder to
score, but closer to how mathematics is actually written, and it yields
substantially more problems per unit of arXiv input (hundreds of lemmas per
arXiv-week versus tens of verifiable problems per month). Our contribution is
therefore not arXiv-based benchmarking in general, but self-contained arXiv
\emph{lemma proof} generation and its natural-language evaluation. Formal-language
benchmarks such as miniF2F \citep{zheng2021minif2f}, FIMO \citep{liu2023fimo},
Lean-Workbook \citep{ying2024lean} and PutnamBench \citep{tsoukalas2024putnambench}
address a complementary, mostly pre-research setting. An extended discussion is
given in Appendix \ref{app:relatedworks}.

\section{Problem Definition: Context Dependence in Extracted Theorems}

Theorem statements in research preprints do not exist within a vacuum but tightly depend upon the definitions, objects and choices introduced previously in the paper. This raises the following question for a research-level mathematics benchmark: make long-context retrieval and contextual understanding part of an \emph{end-to-end task}, or restrict the measured capability clearly to research-level \emph{theorem proving}?

In this work, we opt for the second choice of removing context dependence by retrieving definitions and making statements self-contained. This allows us to provide a more targeted measure for theorem proving for which few metrics exist (for long-context understanding tasks in natural language processing, on the other hand, see for instance \citep{bai2024longbench}), facilitates adoption and, at the same time, makes the pipeline useful for reinforcement learning prompt dataset construction.

What constitutes a \emph{self-contained} theorem? In this work, we treat common undergraduate notions (e.g. the fields of rational, real and complex numbers, rings, modules and vector spaces, continuity, differentiability) as a known background corpus, and require notions and objects specific to the research conducted in the preprint to be retrieved. 
To give examples, the following lemmas are deemed \emph{self-contained}, \emph{not self-contained} and \emph{self-contained given the retrieved context}.

\examplebrace{Example 1\\(self-contained)}{
\begin{lemma}[\citep{de2026objective}, Lemma 13]
  Suppose that
  \begin{equation*}
       a u \leq b + c \log(u),
  \end{equation*}
  for some $a,c>0$, a scalar $b \in\R$ and $u\ge1$. Then
  \begin{equation*}
      u \leq \max \left\{e^{b/c}, \frac{4c^2}{a^2} \right \}.
  \end{equation*}
\end{lemma}
}

\examplebrace{\hspace{0pt}Example 2\\\hspace{0pt}(not self-contained)}{
\begin{lemma}[\citep{washburn2026uniqueness}, Lemma 2.1]
\leavevmode\\
  If \(F\) is reciprocal, then \(G\) and \(H\) are even.
\end{lemma}
}

\examplebrace{\hspace{0pt}Example 3\\\hspace{0pt}(self-contained \\ \hspace{0pt} given \\ retrieved context)}
{\begin{definition} \label{def:context}
Consider the following retrieved context definitions.
  \begin{itemize}
    \item Reciprocal cost: $F:\Rp\to\R$ with $F(x)=F(x^{-1})$\\ for all $x>0$;
          it is normalized if $F(1)=0$.
    \item $G:\R\to\R$ defined by $G(t):=F(e^t)$.
    \item $H:\R\to\R$ defined by $H(t):=G(t)+1=F(e^t)+1$.
  \end{itemize}
\end{definition}
\begin{lemma}[\citep{washburn2026uniqueness}, Lemma 2.1]
\leavevmode\\
  If \(F\) is reciprocal, then \(G\) and \(H\) are even.
\end{lemma}
}

Another example of a self-contained lemma retrieved by GPT-5 is given in \ref{Lemma: Hilbert Schmidt operator} in Appendix \ref{app:examples}

\section{Method}
\label{sec:method}
\subsection{Data extraction}
\label{sec:extract}

While the research articles produced by mathematicians and stored in the arXiv represent a large source of research-level theorems, these theorems are often not self-contained:  some assumptions needed to prove the statements may be scattered in the rest of the paper and only implicitly assumed in the actual statement. This is a notable reason why autoformalization of arXiv papers is challenging \cite{becker2025blueprint}, just as their use in a theorem proving benchmark. To tackle this, we design a language-model based pipeline to identify the assumptions missing or implicitly assumed for each statement and extract them from the articles to make the statements self-contained.

Because our goal is to retrieve lemmas, 
the pipeline applies first a regex-based extractor to the LaTeX source and captures various lemma environments\footnote{From the second iteration of the benchmark in Feb. 2026, we also discard all lemmas that bear a reference, indicating that they are already known in the literature and only recalled in the paper.}
We focus on lemmas because they are typically more granular and, on average, more tractable for current LLM provers; extending to other LaTeX theorem environments is, of course, straightforward.
\newpage
To obtain the definitions and assumptions needed to make a given statement self-contained, we investigate two complementary modes:

\vspace{-\baselineskip}
\paragraph{Full-context retrieval.}
Given a targeted statement $S$,
we provide the full content of the article preceding this statement to an LLM which is then asked to retrieve all definitions $D$ and assumptions $H$ that would be relevant so that $S$ is self-contained given $D$ and $H$.
Given a sufficiently capable model for long-context retrieval and reasoning, this approach should be able to fully resolve all dependencies that can be resolved from the preprint alone.

\vspace{-\baselineskip}
\paragraph{Vector Retrieval.}
For each targeted statement, a first LLM is asked to 
enumerate the non-trivial \emph{objects} appearing in the statement -- i.e., symbols, structures, or named results that are neither standard nor defined in the statement. We split the content of the article preceding the statement in paragraphs and, for each object, we use a regex to list all paragraphs containing it. 

Separately, we compute dense embeddings of all paragraphs preceding the lemma and the literal query \texttt{"Definition of <object>"} using the \texttt{text-embedding-3-small} model in the OpenAI API and compute cosine similarity scores with ChromaDB. Concretely, we keep the top-$k$ hits ($3 \leq k \leq 12$) whose similarity exceeds a threshold of $\tau=0.75$.

Concatenating all paragraphs identified by the regex and selected by the retriever, for each non-trivial object in the statement, results in a context that can be passed to a second LLM call to extract the exact definitions and assumptions required to make the lemma fully self-contained.

The \textbf{full-context} method requires by design a larger number of tokens than vector retrieval in that it provides a much larger context. In our setup, full-context extraction needed about 5-10 times more tokens than vector retrieval depending on the model. However, it has a higher accuracy in retrieving the assumptions of the different lemmas as presented in the following subsection \ref{ssec: self-containement test}.
\vspace{-0.25\baselineskip}

\subsection{Self-Containedness Judge} \label{ssec: self-containement test}

After lemmas have been complemented with extracted assumptions, we filter for true self-containedness\footnote{From the second iteration of this benchmark in Feb. 2026, we first apply a deterministic filter to ensure that no external citations or undesirable environments are referenced in the extracted assumptions.} and we then use an LLM-as-a-judge as an additional guardrail.
For this, the lemma statement and its extracted assumptions are passed to the LLM with the task of assessing whether with the additional context, the lemma is self-contained, i.e. dependency-free except for notions deemed standard or classical. The model outputs a reasoning trace alongside a binary final judgment.

\vspace{-\baselineskip}
\paragraph{Evaluation} We evaluate this extraction part of the pipeline up-front, settling on optimal choices for the final benchmark used for model evaluations in Section~\ref{sec:results}. Specifically, we conduct human annotations as well as model cross-validations to assess the LLM-based evaluation of self-containedness.

\vspace{-\baselineskip}
\paragraph{Human evaluation} A significant proportion  of the lemmas obtained were reviewed by human experts to check whether they were indeed self-contained and to evaluate the capabilities of the LLM-as-a-judge. 
Depending on the mode, between 75.5\% and 96.5\% of the lemmas vetted by the LLM-as-a-judge were indeed self-contained according to human mathematicians. The full results can be found in Table \ref{tab:eval}.

\begin{table}[htbp]
\centering
\resizebox{\columnwidth}{!}{%
\begin{tabular}{lllcc}
\toprule
\textbf{Extraction model} & \textbf{Mode} & \textbf{Judge model} & \textbf{SC (judge) (\%)} & \textbf{PPV (\%)} \\
\midrule
Gemini 2.5 Pro & Full context     & Gemini 2.5 Pro & 67.4 & 96.5 \\
Gemini 2.5 Pro & Vector retrieval & Gemini 2.5 Pro & 36.9 & 88.4 \\
GPT-5          & Full context     & GPT-5          & 78.5 & 90.0 \\
GPT-5          & Vector retrieval & GPT-5          & 49.4 & 85.0 \\
GPT-4          & Full context     & o3             & 47.6 & 78.1 \\
GPT-4          & Vector retrieval & o3             & 19.4 & 75.5 \\
GPT-4          & Full context     & GPT-4          & 64.4 & 88.2 \\
GPT-4          & Vector retrieval & GPT-4          & 25.8 & 80.2 \\
\bottomrule
\end{tabular}
}
\caption{\textbf{Evaluation of extraction modes and extraction models.} We report self-containedness (SC) proportions for 376 lemmas extracted with both modes and three different models, judged by LLM-as-a-judge using the same model. For GPT-4, we additionally use o3 as a stronger judge to measure the discrepancy. PPV (positive predictive value) is the proportion of lemmas judged self-contained that are confirmed by human evaluators.}
\label{tab:eval}
\end{table}

\paragraph{Cross-model evaluation}
To further evaluate the robustness of self-containedness judgments, we performed a symmetric cross-validation: judging lemmas extracted by GPT-5 and Gemini 2.5 Pro by both models to assess the presence of systematic failure modes that each individual model cannot self-correct. The experiment uses full-context extraction and yields a cross-model confusion matrix for each of the extracted datasets.

\begin{table}[h!]
\centering
\resizebox{\columnwidth}{!}{%
\begin{tabular}{llrrr}
\toprule
\textbf{Extracted by} & \textbf{Judged by} & \textbf{Total} & \textbf{Self-contained} & \textbf{Not self-contained} \\
\midrule
Gemini 2.5 Pro & GPT-5         & 254 & 183 (72\%) & 71 (28\%) \\
\multicolumn{2}{l}{\quad\textit{confirmed by human judge}} & -- & 183 (100\%) & 17 (24\%) \\
\midrule
GPT-5          & Gemini 2.5 Pro & 296 & 293 (99\%) & 3 (1\%) \\
\bottomrule
\end{tabular}
}
\caption{Cross-validation between Gemini 2.5 Pro and GPT-5 for detecting self-contained lemmas. Percentages on the second row refer to confirmation rates within each cell above.}
\label{tab:crossval}
\end{table}

We performed a cross-model validation of mathematical self-containedness between Gemini 2.5 Pro and GPT-5. Among the 254 lemmas that were extracted by Gemini and labeled as self-contained by Gemini, GPT-5 confirmed 72\% and rejected 28\%. Manual inspection showed that 100 \% of those lemmas judged self contained by GPT-5 are truly self contained and 24\% of those judged by GPT5 as not self-contained lemmas were in fact truly self-contained, indicating that GPT-5 acts as a conservative validator, producing almost no false positives but a non-negligible number of false negatives. 
These results show that the two models implement substantially different implicit criteria for self-containedness. Given that GPT-5 seemed to behave as a higher-precision classifier we chose GPT-5 as main extractor for the benchmark.

\paragraph{Results}
Full-context retrieval succeeds in complementing a larger number of lemmas by all required definitions and hypotheses to make them self-contained according to LLM-as-judge evaluation, and human evaluation. 
Accuracy and volume of self contained lemma rightfully retrieved are presented in Figure~\ref{fig:retrieving}. Overall more than $60\%$ of the lemmas identified by the regex in the initial step are successfully self-contained by the pipeline. Therefore we used full-context retrieval by default when establishing the benchmark.

\section{Model Evaluation}
\label{sec:results}

\vspace{-0.5\baselineskip}

\label{sec:gen-judge}
We use the resulting benchmark to evaluate frontier LLMs on research-level lemma proving tasks. Specifically, we use a structured, step-wise format for proof presentation and independent LLM evaluators for judging correctness.

\vspace{-\baselineskip}
\paragraph{Generator.}
The candidate LLM  \emph{prover} is queried to produce a structured proof in numbered steps, with subgoals and the proof explicitly stated.  
We enforce minimal constraints: no invocation of undeclared external references; hypotheses must be cited at the moment of use; known theorems must be clearly announced when applied.  

\vspace{-\baselineskip}
\paragraph{Independent LLM-as-a-judge.}
The lemma, its assumptions and the generated proof are presented to a separate LLM judge, which evaluates the overall proof.

\paragraph{Domain Stratification}
We stratify by domain, following the \texttt{math.*} classification on the arXiv 
(AG, AP, PR, NT, GR, MG, OC, FA, etc.).  
Our dataset was built from articles uploaded 
within a single week. We present here three iterations of the benchmark: one performed in September 2025 corresponding to articles submitted the last week of  August 2025, another one in February 2026 corresponding to articles submitted the first week of February 2026, followed by one batch collecting articles submitted on the last week of April 2026. \footnote{We use the most recent article version regardless of the original submission} .
For each domain, we sampled 1 to 5 articles from the target week, with the exact selection guided by a mix of random choice and manual curation to ensure topical diversity.
Table~\ref{tab:domains} reports the distribution for the lemma extracted from the last week of August 2025. 
\vspace{-0.5\baselineskip}
\paragraph{Evaluation: Metrics, Protocols, and Tables}

To evaluate the pipeline and compare models, we define several metrics covering extraction quality, proof validity, and human assessment.

\paragraph{Main metrics.}
\begin{itemize}
  \item \textbf{Self-Containment Pass} (SC)(\%): percentage of lemmas validated as self-contained 
  (\texttt{self\_contained=1}). This reflects the robustness of the extraction pipeline and coherence judge.  
  In our experiments, values above $60\%$ are considered a satisfactory quality threshold.
  \item \textbf{Precision (PPV) — true self-contained pass (\%):} Proportion of lemmas flagged as self-contained that are confirmed as such by human validation.
  This is the \emph{precision} (positive predictive value) of the self-contained detector, i.e., how much we can trust the set of “pass” lemmas produced by a given method.

  \item \textbf{Proof Acceptance}: proportion of self-contained lemmas that were correctly proven by the LLM, tested at pass@1.
  \item \textbf{Human Confidence Score}: 
  percentage of proof judged correct by the LLM-as-a-judge that are also judged correct by human experts. 
\end{itemize}

For the week of August 2025, we extracted $376$ lemmas.  
On average across models, about $240$ (64\%) of these are judged self-contained after the pipeline. On this benchmark of 240 elements, 7-12.3\% are correctly proven  depending on the model tested.
The evaluation of the human score was performed on 10-20 such proofs per model by expert mathematicians (for a total of 44 proofs reviewed), chosen from reviewers of journals from different fields of mathematics. The proof judged correct by the LLM-as-a-judge were vetted correct by human standards in $67-83\%$ of the cases

\begin{table}[ht]
\centering
\resizebox{\columnwidth}{!}{%
\setlength{\tabcolsep}{3pt} 
\renewcommand{\arraystretch}{1.0}

\begin{tabular}{llcc}
\toprule
\textbf{Source} & \textbf{Gen$\to$Judge} & \textbf{Proof Accept.\ (\%)} & \textbf{Human confidence score (\%)} \\
\midrule
Full-ctx (GPT-5)   & GPT-5$\to$GPT-5      & 12.3 & 83 \\
Full-ctx (Gem-2.5) & Gem-2.5$\to$Gem-2.5  & 7 & 67 \\
Full-ctx (GPT-5)   & DeepSeek-R1$\to$DeepSeek-R1 & 11.9 & 80 \\
\bottomrule
\end{tabular}
}
\vspace{0.3em}
\footnotesize 
\newline 
\caption{\label{tab:results-compact}\textbf{Proof quality - September 2025} by source and generator $\to$ judge for the first iteration (week of August 2025, experiment performed in September 2025).}
\end{table}

For the first week of February 2026, we select 100 preprints from which we extract a total of 677 lemmas appearing in 81 preprints using GPT-5. Out of these, 358 lemmas (52.8\%) were judged as self-contained by GPT-5 after full-context definition and hypothesis extraction by GPT-5 and deterministic filtering. On the resulting benchmark 
106 lemmas (29.6\%) are correctly proven according to the most critical judge (GPT-5) with the best model (also GPT-5). 

For the batch completed on the last week of April, we select 31 preprints for a total of 287 lemmas using GPT-5.5. Out of these, 120 lemmas remain, of which 49 lemmas (40.8\%) are correctly proven according to the most critical judge (GPT-5.5) with the best model (also GPT-5.5).
The results for the latest iteration of the benchmark in April 2026 are shown in Table \ref{tab:results-compact-april-2026-whole-proof} and the evolution with time of the performances of the best model (GPT-5 / GPT-5.5) is shown in figure \ref{fig:evolution}. 
While evaluation of the human score is not yet available for these iterations of the benchmark, with the improving capabilities of general public LLMs in mathematics it seems reasonable to expect that the confidence score would be at least of the order of the one of September 2025.

\begin{table}[t]
\centering
\setlength{\tabcolsep}{3pt}
\footnotesize
\begin{threeparttable}
\begin{tabular}{llrr}
\toprule
\textbf{Generator} & \textbf{Judge} & \textbf{\# Acc.} & \textbf{Rate (\%)} \\
\midrule
GPT-5.5      & GPT-5.5      & \textbf{49}  & \textbf{40.8} \\
             & Gemini 3.1   & 79  & 65.8 \\
             & DeepSeek-R1   & \textbf{62}  & \textbf{51.7} \\
\midrule
Gemini 3.1   & GPT-5.5      & 37  & 30.8 \\
             & Gemini 3.1   & \textbf{83}  & \textbf{69.2} \\
             & DeepSeek-R1   & 53  & 44.2 \\
\midrule
DeepSeek-R1\tnote{a} & GPT-5.5  & 30  & 27.3 \\
             & Gemini 3.1   & 57  & 51.8 \\
             & DeepSeek-R1   & 45  & 40.9 \\
\bottomrule
\end{tabular}
\begin{tablenotes}[flushleft]
\footnotesize
\item[a] DeepSeek-R1 failed to produce a parseable proof for 10 of the 120 lemmas; acceptance rates for this generator are computed over the 110 successfully generated proofs.
\end{tablenotes}
\end{threeparttable}
\caption{\textbf{Proof quality -- Batch of April 2026.} Three LLMs as generators, three as whole-proof judges. All proofs use filtered lemmas extracted by GPT-5.5. Experiment performed in May on lemmas extracted on the last week of April.}
\label{tab:results-compact-april-2026-whole-proof}
\end{table}


\begin{figure}[!t]
\centering
\begin{tikzpicture}
\begin{axis}[
    width=\columnwidth,
    height=5cm,
    xlabel={Batch month},
    ylabel={Correct proofs $(\%)$},
    xtick={1,24,35},
    xticklabels={Aug.\ 2025, Feb.\ 2026, Apr.\ 2026},
    ymin=0,
    ymax=50,
    grid=both,
    mark size=2.5pt,
]
\addplot[thick, mark=*] coordinates {
    (1, 12.3)
    (24, 29.6)
    (35, 40.8)
};
\end{axis}
\end{tikzpicture}
\caption{\label{fig:evolution}\textbf{Evolution of proof quality of the best performing model with time:} GPT-5 in Aug.\ 2025 and Feb.\ 2026, and GPT-5.5 in April 2026, according to the most critical judge (same model).}
\end{figure}

\paragraph{Step-by-step evaluation} We also performed an alternative experiment where the LLM-as-a-judge evaluates each step of the overall proof with a binary verdict (\texttt{True}, \texttt{False}), and labels the whole proof as correct if and only if all steps are correct. This mode of evaluation is likely too strict as the LLM-as-a-judge has a tendency to reject proofs with imprecisions that would be insignificant to humans (see Appendix G table \ref{tab:step-by-step-no-cite} \& \ref{app:gpt5-too-strict} for an example) but can be used as a stricter lower bound and compare the models between each other.
The results on Feb. 2026 can be found in Table \ref{tab:results-compact-2026}. While the opinion of different LLMs judges may differ, GPT-5 is still seen by all judges as the best performing model.

\begin{table}[ht]
\centering
\resizebox{\columnwidth}{!}{%
\setlength{\tabcolsep}{6pt}
\renewcommand{\arraystretch}{1.1}
\begin{tabular}{lllrr}
\toprule
\textbf{Source} & \textbf{Generator} & \textbf{Judge} & \textbf{\# Accepted} & \textbf{Rate (\%)} \\
\midrule
Full-ctx (GPT-5) & GPT-5         & GPT-5           & \textbf{55}  & \textbf{15} \\
                 &               & Gemini 3 pro    & \textbf{67}  & \textbf{19} \\
                 &               & Claude Opus 4.5 & \textbf{125} & \textbf{35} \\
\midrule
Full-ctx (GPT-5) & Gemini 3 pro  & GPT-5           & 12  & 3  \\
                 &               & Gemini 3 pro    & 57  & 16 \\
                 &               & Claude Opus 4.5 & 104 & 29 \\
\midrule
Full-ctx (GPT-5) & Opus 4.5      & GPT-5           & 6   & 2  \\
                 &               & Gemini 3 pro    & 50  & 14 \\
                 &               & Claude Opus 4.5 & 60  & 17 \\
\bottomrule
\end{tabular}
}
\caption{\label{tab:results-compact-2026}\textbf{Proof quality -- February 2026.} Acceptance counts and rates over 358 candidate lemmas, with three LLMs as generators and three LLMs as whole-proof judges. All proofs generated under full-context conditions from lemmas extracted by GPT-5.}
\end{table}

\section{Discussion}
This paper presents a new approach to evaluate LLM mathematical capabilities in natural language. 
It consists in a live benchmark built directly from the latest research-level mathematical results proved by mathematicians on the arXiv. A key advantage is that it reduces the risk of data contamination over time, and avoids treating narrow, convenience-curated problem collections, whether competition exercises or research questions selected primarily for easy verifiability, as a proxy for mathematical research, where, by definition, data is relatively scarce (in contrast with mathematical competition exercises where in-distribution synthetic data has been massively generated in the last two years). 
The price to pay for this generality is the use of an LLM-as-a-judge to evaluate the qualities of the self-contained lemmas and the quality of the proof provided. 
Nevertheless, our evaluation by human mathematicians shows the reliability of LLM-as-a-judge in evaluating both the self-containedness of a statement and the validity of a proof. This can be explained as it is usually easier for a model to evaluate the validity of a proof than to actually produce a proof. 
The theorem proving capabilities of LLMs and their reliability as judges should improve as models get better in reasoning. We believe this is an additional motivation to use the type of benchmark introduced in this article. The quantitative results reported here reveal informative trends, and should be considered preliminary given the limited sample size (a few hundred problems) and the small number of reviewers (a dozen human mathematicians), and as an incitation to design such benchmarks directly grounded in the latest research without restricting to specific problems. We plan to extend the evaluation protocol with more reviewers and more examples to improve the statistical robustness of the analysis.

\section*{Code Availability}
The initial datasets and code are publicly available at
{\faGithub}\hspace{0.3mm}\href{https://github.com/ahayat16/LemmaBench/tree/main/data}{LemmaBench}.

\section*{Acknowledgements}
The authors would like to thank Etienne Bernard, Urbain Vaes, Vincent Boulard, Alexandre Krajenbrinck,  Ilan Zysman, David Renard, Baptiste Seraille, Thomas Cavallazzi, Hector Bouton, Gabriel Synnaeve, Claudio Dappiaggi.\\
This work was granted access to the HPC resources of IDRIS under the allocation 2024-AD011014527R1 made by GENCI. AH and AP were supported by the ANR-Tremplin StarPDE ANR-24-ERCS-0010, the PEPS Project JCJC 2025 from CNRS - INSMI and the Hi! PARIS and ANR/France 2030 program (ANR-23-IACL-0005).

\FloatBarrier
\newpage
\bibliographystyle{icml2026}
\bibliography{biblio.bib}
\clearpage
\newpage 
\appendix

\section{Extended Related Work}
\label{app:relatedworks}
With the increasing capabilities of LLM and LLM-based approaches in mathematics, several benchmarks have been developed in recent years. In natural language, one can cite for instance GSM8k \cite{cobbe2021training}, MathQA \cite{amini2019mathqa} or MATH \cite{hendrycks2021measuring} which provided middle to high-school level questions. These benchmarks were often used to evaluate LLM capabilities to the point of being the staple of evaluation for mathematical abilities. This was until most models were able to reach very high to nearly perfect accuracies \cite{gao2024omni}. These datasets regularly faced debate on contamination issues (see for instance \cite{xu2024benchmarking}).
Other datasets such as
OlymMATH \cite{sun2025challenging}
later proposed problems up to college level coming from multilingual sources of exercises. A very large dataset of competition problems was created by the Numina project \cite{numina_math_datasets}.  OmniMath introduced challenging Olympiad level problems \cite{gao2024omni}. More recent examples like MathOdyssey \cite{fang2025mathodyssey} introduced a benchmark, still on high school to undergraduate level, curated by experts to limit the risk of data contamination. Nevertheless, given the accessibility of the dataset, contamination issues may arise in future generations of LLMs. FrontierMath \cite{frontiermath} had a similar expert-curated problems approach and to avoid any contamination issues, the dataset is purely closed-source. A larger similar benchmark was released with more than 400 problems curated by 64 mathematicians \cite{son2026soohak}, closed-source for the year 2026. Other benchmarks were recently released, such as First Proof \cite{abouzaid2026first} where the statements of ten expert-curated questions with no publicly available solutions are provided during one-week for a one time evaluation of LLMs. Leipzig Tier-4 Benchmark present 102 research-level problems curated by mathematicians in open-source \cite{leipzig_benchmarks_2026}.
RealMath \cite{zhang2025realmath} extract or convert problems from arXiv papers with verifiable answers, so as to focus on real human research too, see also \cite{arxivmath2026} that has a monthly update. While the approach is similar to ours, in particular for \cite{arxivmath2026}\footnote{The first iteration of \cite{arxivmath2026} was released while we were preparing a revision of this manuscript.}, it differs in that it restricts itself to problems that can be inherently formulated with a verifiable answers. This explains why it gives rise to less problems (40 problems over 2 months of arXiv against 400 lemmas for 2 weeks of arXiv) and why the performances of state-of-the-art LLMs differ (60\% on the benchmark of \cite{arxivmath2026} compared to around 15\% for our). EternalMath \cite{ma2026eternalmath} is yet another similar approach where the questions are selected so as to propose an executable benchmark that can be code-verified.

It is worth mentioning that other benchmarks exist in formal language, also focusing on middle school to undergraduate level problems such as miniF2F \cite{zheng2021minif2f}, FIMO \cite{liu2023fimo}, Lean-Workbook \cite{ying2024lean},
PutnamBench \cite{tsoukalas2024putnambench}.

\section{Table of percentage results relative to the total lemmas retrieve for each method (A) and (B) and LLM}

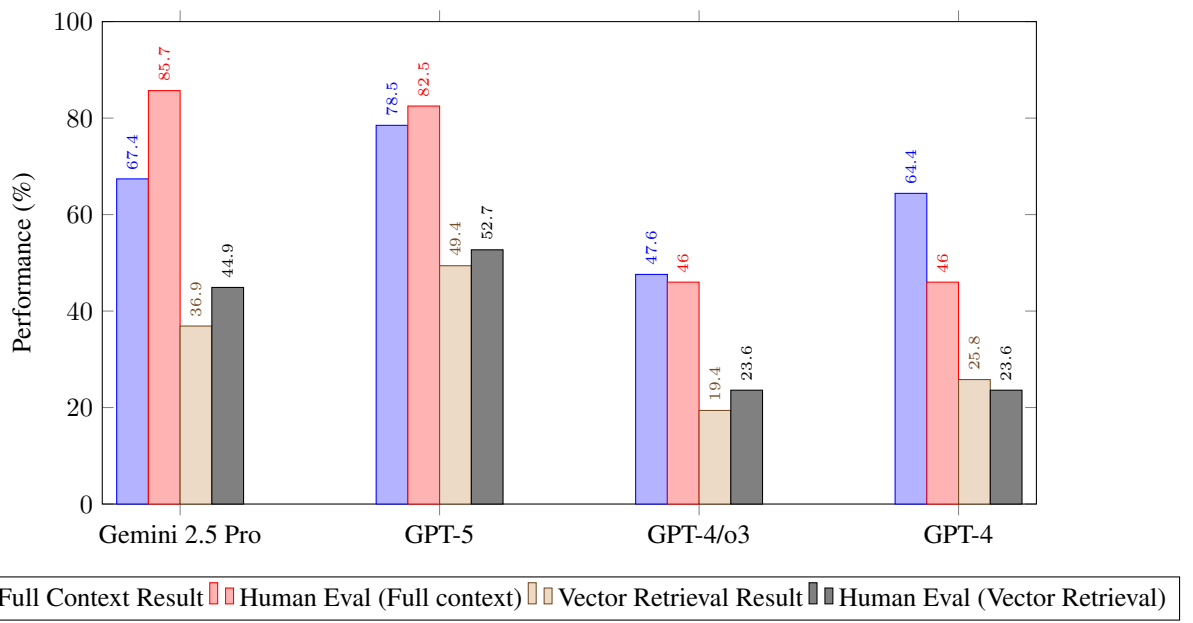
\begin{figure}[h!]
\centering
\begin{tikzpicture}
\begin{axis}[
    ybar=0pt,
    bar width=5pt,
    width=\columnwidth, height=6cm,
    ylabel={Performance (\%)},
    ylabel near ticks,
    symbolic x coords={Gemini 2.5, GPT-5, GPT-4/o3, GPT-4},
    xtick=data,
    x tick label style={font=\footnotesize, rotate=20, anchor=east},
    legend style={at={(0.5,-0.3)}, anchor=north, legend columns=2, font=\footnotesize},
    ymin=0, ymax=100,
    nodes near coords,
    every node near coord/.append style={font=\tiny, rotate=90, anchor=west},
]
\addplot coordinates {(Gemini 2.5,67.4) (GPT-5,78.5) (GPT-4/o3,47.6) (GPT-4,64.4)};
\addplot coordinates {(Gemini 2.5,85.7) (GPT-5,82.5) (GPT-4/o3,46) (GPT-4,46)};
\addplot coordinates {(Gemini 2.5,36.9) (GPT-5,49.4) (GPT-4/o3,19.4) (GPT-4,25.8)};
\addplot coordinates {(Gemini 2.5,44.9) (GPT-5,52.7) (GPT-4/o3,23.6) (GPT-4,23.6)};
\legend{Full Ctx, Human (Full), Vector Ret., Human (Vec.)}
\end{axis}
\end{tikzpicture}
\caption{Percentage results relative to the total lemmas retrieved for both extraction modes (full context / vector retrieval) and models}
\label{fig:retrieving}
\end{figure}

\clearpage
\section{Breakdown of lemmas per domain of mathematics}
\label{app:strat}

\begin{table}[htbp] 

\centering
\begin{tabular}{llrr}
\toprule
\textbf{Domain code} & \textbf{Domain label} & \textbf{\# Articles} & \textbf{\#Lemmas} \\
\midrule
cs.IT   & Computer Science           & 3 & 24 \\
math-ph & Mathematical Physics       & 2 & 36 \\
math.AP & Analysis of PDEs           & 3 & 46 \\
math.CO & Combinatorics              & 5 & 60 \\
math.CT & Category Theory            & 1 & 22 \\
math.DG & Differential Geometry      & 1 & 9  \\
math.DS & Dynamical Systems          & 2 & 18 \\
math.FA & Functional Analysis        & 1 & 15 \\
math.GR & Group Theory               & 1 & 5  \\
math.GT & Geometric Topology         & 1 & 6  \\
math.MG & Metric Geometry            & 1 & 2  \\
math.NA & Numerical Analysis         & 2 & 13 \\
math.NT & Number Theory              & 2 & 2  \\
math.OC & Optimization and Control   & 4 & 19 \\
math.PR & Probability                & 3 & 51 \\
math.QA & Quantum Algebra            & 1 & 2  \\
math.RT & Representation Theory      & 1 & 16 \\
math.SG & Symplectic Geometry        & 1 & 2  \\
math.ST & Statistics Theory          & 2 & 28 \\
\bottomrule
\end{tabular}
\caption{Domain distribution in the dataset (one-week snapshot, August).}
\label{tab:domains}
\end{table}

\section{System Overview: Directed Acyclic Graph of the Pipeline}

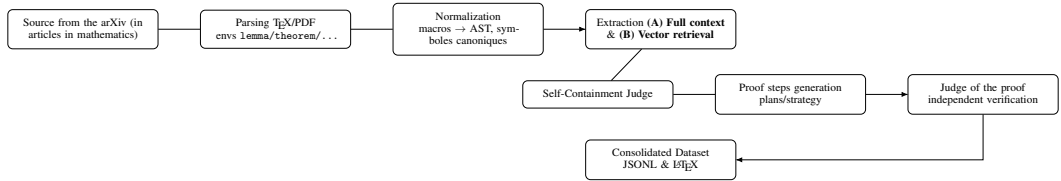
\begin{figure}[ht]
\centering
\resizebox{\linewidth}{!}{%
\begin{tikzpicture}[
  node distance=9mm and 12mm,
  box/.style={draw, rounded corners, align=center, inner sep=2.5mm, minimum width=35mm, text width=38mm},
  every node/.style={font=\footnotesize}
]
\node[box] (harvest) {Source from the arXiv (in articles in mathematics)};
\node[box, right=of harvest] (parse) {Parsing \TeX/PDF \\ \footnotesize envs \texttt{lemma/theorem/...}};
\node[box, right=of parse] (normalize) {Normalization \\ \footnotesize macros $\to$ AST, symboles canoniques};
\node[box, right=of normalize] (extract) {Extraction \textbf{(A) Full context} \\ \footnotesize \& \textbf{(B) Vector retrieval}};
\node[box, below=of extract, xshift=-18mm] (scjudge) {Self-Containment Judge \\};
\node[box, right=of scjudge] (gen) {Proof steps generation \\ \footnotesize plans/strategy};
\node[box, right=of gen] (judge) {Judge of the proof\\ \footnotesize independent verification};
\node[box, below=of scjudge, xshift=18mm] (ds) {Consolidated Dataset \\ \footnotesize JSONL \& \LaTeX};

\draw[-{Latex}] (harvest)--(parse)--(normalize)--(extract);
\draw[-{Latex}] (extract)--(scjudge)--(gen)--(judge);
\draw[-{Latex}] (judge) |- (ds);
\end{tikzpicture}%
}
\caption{DAG : input $\to$ parsing $\to$ normalization $\to$ extraction (double mode) $\to$ self contained $\to$ proof steps $\to$ judge proofs  $\to$ dataset.}

\end{figure}

\clearpage

\section{Algorithms (Pseudo-Code)}

\begin{algorithm}[htbp]

\caption{Extraction \& Self-Containment (Full context / vector retrieval)}
\begin{algorithmic}[1]
\REQUIRE Article \TeX/PDF, mode $\in \{\texttt{full\_context}, \texttt{vector\_retrieval}\}$
\FOR{each \texttt{lemma} environment detected}
  \STATE $S \gets \textbf{statement}(\texttt{lemma})$
  \IF{mode = full\_context}
    \STATE $C \gets$ all paragraphs before lemma
  \ELSE
    \STATE $C \gets \textbf{retrieve\_topk}(S)$ paragraphs \COMMENT{ embeddings}
  \ENDIF
  \STATE $(\hat H, \hat D) \gets \textbf{LLM\_extract}(S,C)$
  \STATE $ok \gets \textbf{SelfContainmentChecks}(S,\hat H,\hat D)$
  \IF{$ok$}
    \STATE \textbf{emit} $(S,\hat H,\hat D,\texttt{self\_contained}=1)$
  \ELSE
    \STATE \textbf{emit} $(S,\hat H,\hat D,\texttt{self\_contained}=0)$
  \ENDIF
\ENDFOR
\end{algorithmic}
\end{algorithm}

\begin{algorithm}[htbp]
\caption{Step Generation and Judgment}
\begin{algorithmic}[1]
\REQUIRE $(S,H,D)$ with \texttt{self\_contained} $=1$
\STATE $\mathrm{steps} \gets \textbf{LLM\_prove}(S,H,D)$
\STATE $\mathrm{verdict} \gets \textbf{LLM\_judge\_proof}(\mathrm{whole\text{ } proof},H,D)$

\IF{$\mathrm{verdict\_ok}$}
  \STATE \textbf{emit} $\mathrm{steps}, \mathrm{verdict}$
\ELSE
  \STATE \textbf{emit} $\mathrm{steps}, \mathrm{verdict}$ \COMMENT{partial or rejected}
\ENDIF
\end{algorithmic}
\end{algorithm}

\vspace{1cm}

\section{Example of a retrieved self contained lemma, with proof generated by gpt-5}
\label{app:examples}
\subsection*{Assumptions}

\begin{itemize}
\item $L^2_{\alpha-1}(0,T;U):=\{u:(0,T)\to U\ \text{measurable}\ |\ \int_{0}^{T}(T-s)^{\alpha-1}\|u(s)\|_{U}^{2}\,ds<\infty\}$, with inner product
\[
\langle u,v\rangle_{L^2_{\alpha-1}(0,T;U)}
:=\int_{0}^{T}(T-s)^{\alpha-1}\langle u(s),v(s)\rangle_{U}\,ds.
\]
Assumptions: $\alpha\in(0,1)$, $T>0$, $U=\mathbb{R}^m$.

\item $L^{2}(\mathfrak{S},H;\mu):=\{f:\mathfrak{S}\to H\ \text{measurable}\ |\ \int_{\mathfrak{S}}\|f(\sigma)\|_{H}^{2}\,d\mu(\sigma)<\infty\}$.  
Assumptions: $(\mathfrak{S},\mathcal{F},\mu)$ is a probability space, $H=\mathbb{R}^{n}$.

\item $\mathfrak{L}_{T}:L^2_{\alpha-1}(0,T;U)\to L^{2}(\mathfrak{S},H;\mu)$ defined by
\begin{multline*}
(\mathcal{L}_{T}u)(\sigma):=\int_{0}^{T}(T-s)^{\alpha-1}E_{\alpha,\alpha}\big((T-s)^{\alpha}A(\sigma)\big)\,B(\sigma)\,u(s)\,ds \\
=\int_{0}^{T}\mathcal{K}_{T}(\sigma,s)\,u(s)\,ds,
\end{multline*}
with kernel
\[
\mathcal{K}_{T}(\sigma,s):=(T-s)^{\alpha-1}E_{\alpha,\alpha}\big((T-s)^{\alpha}A(\sigma)\big)B(\sigma)\in\mathbb{R}^{n\times m},
\]
and two-parameter Mittag-Leffler matrix function
\[
E_{\alpha,\alpha}(M):=\sum_{k=0}^{\infty}\frac{M^{k}}{\Gamma(\alpha k+\alpha)}.
\]
\end{itemize}

\begin{lemma} \label{Lemma: Hilbert Schmidt operator}
$\mathcal{L}_{T}: L^2_{\alpha-1}(0,T;U)\rightarrow L^{2}(\mathfrak{S},H;\mu)$ is a Hilbert--Schmidt operator.
\end{lemma}

\bigskip
\hrule
\bigskip

\subsection{Proof steps}
\textbf{Identification of the domain as a standard $L^2$-space.}  
Define the finite measure $\nu$ on $(0,T)$ by $d\nu(s) := (T-s)^{\alpha-1}\,ds$. Then
\[
L^2_{\alpha-1}(0,T;U) = L^2((0,T),\nu;U).
\]
With this notation the operator reads, for $u \in L^2((0,T),\nu;U)$,
\begin{multline*}
(\mathcal{L}_T u)(\sigma) = \int_0^T (T-s)^{\alpha-1} E_{\alpha,\alpha}((T-s)^\alpha A(\sigma)) B(\sigma) u(s)\,ds
\\ = \int_0^T K_T(\sigma,s) u(s)\,ds,
\end{multline*}
where $K_T(\sigma,s) = (T-s)^{\alpha-1} E_{\alpha,\alpha}((T-s)^\alpha A(\sigma)) B(\sigma)$.

\textbf{Rewriting as an integral operator with respect to $\nu$.}  
Since $d\nu(s) = (T-s)^{\alpha-1} ds$, we have $ds = (T-s)^{1-\alpha} d\nu(s)$. Therefore
\[
(\mathcal{L}_T u)(\sigma) = \int_0^T \bigl[E_{\alpha,\alpha}((T-s)^\alpha A(\sigma)) B(\sigma)\bigr] u(s)\,d\nu(s).
\]
Thus $\mathcal{L}_T$ is the integral operator with kernel
\[
\tilde K(\sigma,s) := E_{\alpha,\alpha}((T-s)^\alpha A(\sigma)) B(\sigma) \in \mathbb{R}^{n\times m}.
\]

\textbf{Uniform bound on the Mittag-Leffler matrix factor.}  
Let $M_A := \operatorname{ess\,sup}_{\sigma\in\mathfrak{S}} \|A(\sigma)\|$. Using the series expansion,
\[
\|E_{\alpha,\alpha}((T-s)^\alpha A(\sigma))\|
\le \sum_{k=0}^\infty \frac{(T^\alpha M_A)^k}{\Gamma(\alpha k+\alpha)} =: C_E < \infty.
\]

\textbf{L2-integrability of the kernel $\tilde K$.}  
Identify $HS(U,H)$ with $\mathbb{R}^{n\times m}$ with Frobenius norm. Then
\[
\|\tilde K(\sigma,s)\|_{HS} \le C_E \|B(\sigma)\|_{HS}.
\]
Let $M_B := \operatorname{ess\,sup}_{\sigma\in\mathfrak{S}} \|B(\sigma)\|_{HS} < \infty$. Then
\[
\int_{\mathfrak{S}} \int_0^T \|\tilde K(\sigma,s)\|_{HS}^2\,d\nu(s)\,d\mu(\sigma)
\le C_E^2 \frac{T^\alpha}{\alpha} M_B^2 < \infty.
\]

\textbf{Hilbert--Schmidt property.}  
By the standard result, $\mathcal{L}_T$ is Hilbert--Schmidt and
\[
\|\mathcal{L}_T\|_{HS}^2 \le \frac{T^\alpha}{\alpha} C_E^2 M_B^2.
\]

\textbf{Conclusion.}  
$\mathcal{L}_T: L^2_{\alpha-1}(0,T;U) \to L^2(\mathfrak{S},H;\mu)$ is Hilbert--Schmidt, with
\[
\|\mathcal{L}_T\|_{HS}^2 \le \frac{T^\alpha}{\alpha} C_E^2 \int_{\mathfrak{S}} \|B(\sigma)\|_{HS}^2\,d\mu(\sigma).
\]

\section{Step-by-step judging}
\begin{table}[ht]
  \centering
  \begin{tabular}{lcc}
    \toprule
    Proof model & Valid proofs & Pass rate \\
    \midrule
    gpt-5.5         & 25/120 & 20.8\% \\
    gemini-3.1-pro  & 10/120 &  8.3\% \\
    gpt-5.4-mini    &  4/120 &  3.3\% \\
    \bottomrule
  \end{tabular}
  \caption{Step-by-step proof evaluation (judge: gpt-5.5), on (120 lemmas) judged self contained on the batch of April 2026. A proof passes if all of its steps are judged correct.}
  \label{tab:step-by-step-no-cite}
\end{table}

\section{Over-strictness of the harshest judge when judging step-by-step}
\label{app:gpt5-too-strict}

Across our Feb.~2026 batch, the three judges differ sharply in leniency when judging the proof step-by-step. Aggregated over
all proofs each inspects, \textsc{Opus-4.5} accepts $26.2\%$, \textsc{Gemini-3} accepts
$15.8\%$, and \textsc{GPT-5} accepts only $6.3\%$: \textbf{\textsc{GPT-5} is by far the
strictest judge}, and it is the model most often in the minority when judges disagree. The
example below is representative: \textsc{GPT-5} rejects a \emph{correct} proof because one
step contains a single incorrect \emph{parenthetical aside} that plays no role in the
argument---a flaw \textsc{GPT-5} itself describes as ``not used later''---whereas
\textsc{Gemini-3} and \textsc{Opus-4.5} accept the proof.\footnote{Lemma and proof
extracted from arXiv:2602.05756 (Feb.\ 2026); the proof was generated by \textsc{GPT-5}.}

\paragraph{The lemma.}
Let $(X,G)$ and $(Y,G)$ be equicontinuous minimal actions of a discrete group $G$ on
compact Hausdorff spaces, with uniformities $\mathbb{U}_X,\mathbb{U}_Y$ given by the
neighbourhoods of the diagonals. For an entourage $\delta$ write
$\delta[z]:=\{z' : (z',z)\in\delta\}$, and for a point $z$ and a set $B$ write the
hitting-time set $G(z,B):=\{g\in G : g.z\in B\}$.

\begin{lemma}\label{lem:factor}
For $x\in X$ and $y\in Y$, there exists a factor map $\pi\colon X\to Y$ with $\pi(x)=y$ if
and only if for every neighbourhood $V$ of $y$ there is a neighbourhood $U$ of $x$ with
$G(x,U)\subseteq G(y,V)$.
\end{lemma}

\paragraph{The candidate proof (generated by GPT-5).}
The model produced the proof below, which we submitted to each judge one step at a time.
The single questionable clause---an unjustified, and ultimately unused, assertion in
Step~3---is \hlclause{highlighted}.

{\bfseries\small Candidate proof of Lemma~\ref{lem:factor} (generated by \textsc{GPT-5}).}\par\smallskip
\small
\textbf{Step 0.} We recall the setting. Let $G$ be a discrete group. Let $(X,G)$ and $(Y,G)$ be equicontinuous minimal actions on compact Hausdorff spaces, with $g.x:=\alpha_X(g)(x)$ and $g.y:=\alpha_Y(g)(y)$. Denote by $\mathbb{U}_X$ and $\mathbb{U}_Y$ the uniformities given by the neighborhoods of the diagonals in $X^2$ and $Y^2$, respectively. For a minimal action $(Z,G)$, a point $z\in Z$, and $B\subseteq Z$, define the hitting times set
\[
G(z,B):=\{g\in G;\; g.z\in B\}.
\]
We prove:

Lemma. Let $x\in X$ and $y\in Y$. There exists a factor map $\pi\colon X\to Y$ with $\pi(x)=y$ if and only if for any neighborhood $V$ of $y$ there exists a neighborhood $U$ of $x$ such that $G(x,U)\subseteq G(y,V)$.\par\smallskip
\textbf{Step 1.} Proof of the easy direction ($\Rightarrow$). Assume there is a factor map $\pi\colon X\to Y$ with $\pi(x)=y$, i.e., $\pi$ is continuous, surjective, and equivariant: $\pi(g.x)=g.\pi(x)$ for all $g\in G$, $x\in X$.

Let $V\subseteq Y$ be a neighborhood of $y$. By continuity of $\pi$ at $x$, there exists a neighborhood $U\subseteq X$ of $x$ such that $\pi(U)\subseteq V$. If $g\in G(x,U)$, then $g.x\in U$, so $\pi(g.x)\in V$. By equivariance,
\[
\pi(g.x)=g.\pi(x)=g.y\in V,
\]
hence $g\in G(y,V)$. Therefore $G(x,U)\subseteq G(y,V)$.\par\smallskip
\textbf{Step 2.} Proof of the converse ($\Leftarrow$): stabilizer inclusion and definition on the orbit. Assume that for every neighborhood $V$ of $y$ there exists a neighborhood $U$ of $x$ with
\[
G(x,U)\subseteq G(y,V).
\]
First, we show the stabilizer inclusion $\operatorname{Stab}_G(x)\subseteq \operatorname{Stab}_G(y)$.

Indeed, if $h\in G$ satisfies $h.x=x$, then for any neighborhood $V$ of $y$, choose $U$ given by the hypothesis. Since $h.x=x\in U$, we have $h\in G(x,U)$, and thus $h\in G(y,V)$, so $h.y\in V$. As this holds for every neighborhood $V$ of $y$, it follows that $h.y=y$. Hence
\[
\operatorname{Stab}_G(x)\subseteq \operatorname{Stab}_G(y).
\]

Define a map $\varphi$ on the dense orbit $G.x\subseteq X$ by
\[
\varphi(g.x):=g.y\quad\text{for }g\in G.
\]
This is well-defined: if $g_1.x=g_2.x$, then $g_2^{-1}g_1\in \operatorname{Stab}_G(x)\subseteq \operatorname{Stab}_G(y)$, hence $g_1.y=g_2.y$.

Moreover, $\varphi$ is $G$-equivariant on $G.x$: for $h\in G$,
\[
\varphi\bigl(h.(g.x)\bigr)=\varphi((hg).x)=(hg).y=h.(g.y)=h.\varphi(g.x).
\]\par\smallskip
\textbf{Step 3.} Uniform continuity of $\varphi$ on $G.x$. We show that $\varphi\colon G.x\to Y$ is uniformly continuous with respect to the uniformities $\mathbb{U}_X$ and $\mathbb{U}_Y$ restricted to $G.x$ and $Y$, respectively.

Fix $\varepsilon\in \mathbb{U}_Y$. By equicontinuity of $(Y,G)$, there exists $\delta_Y\in \mathbb{U}_Y$ such that for all $(y_1,y_2)\in \delta_Y$ and all $r\in G$,
\[
(r.y_1,r.y_2)\in \varepsilon.
\]
Let $V:=\delta_Y[y]=\{y'\in Y;\; (y',y)\in \delta_Y\}$, a neighborhood of $y$ with \hlclause{$V\times V\subseteq \delta_Y$}.

By the hypothesis, choose a neighborhood $U$ of $x$ such that $G(x,U)\subseteq G(y,V)$.

Choose an entourage $E\in \mathbb{U}_X$ with $E[x]\subseteq U$. By equicontinuity of $(X,G)$, there exists $\eta_X\in \mathbb{U}_X$ such that for all $(u,v)\in \eta_X$ and all $r\in G$,
\[
(r.u,r.v)\in E.
\]

We claim: if $g,h\in G$ satisfy $(g.x,h.x)\in \eta_X$, then $(\varphi(g.x),\varphi(h.x))=(g.y,h.y)\in \varepsilon$.

Indeed, applying $r=h^{-1}$ to $(g.x,h.x)\in \eta_X$, we get
\[
\bigl(h^{-1}g.x,\,h^{-1}h.x\bigr)=\bigl(h^{-1}g.x,\,x\bigr)\in E,
\]
thus $h^{-1}g.x\in E[x]\subseteq U$, i.e., $h^{-1}g\in G(x,U)\subseteq G(y,V)$. Hence $h^{-1}g.y\in V$, which means $(h^{-1}g.y,y)\in \delta_Y$. Applying $r=h$ and the choice of $\delta_Y$, we obtain
\[
(g.y,h.y)=\bigl(h.(h^{-1}g.y),\,h.y\bigr)\in \varepsilon.
\]
This establishes the desired uniform continuity of $\varphi$ on $G.x$.\par\smallskip
\textbf{Step 4.} Extension to a continuous map $\pi\colon X\to Y$. Since $G.x$ is dense in $X$ (by minimality) and $Y$ is compact Hausdorff (hence complete and Hausdorff as a uniform space), the uniformly continuous map $\varphi\colon G.x\to Y$ extends uniquely to a uniformly continuous (hence continuous) map
\[
\pi\colon X\longrightarrow Y.
\]
By construction, $\pi|_{G.x}=\varphi$, so in particular $\pi(x)=\varphi(x)=y$.\par\smallskip
\textbf{Step 5.} Equivariance and surjectivity of $\pi$. For each fixed $h\in G$, the maps $X\to Y$, $z\mapsto \pi(h.z)$ and $z\mapsto h.\pi(z)$ are continuous and agree on the dense set $G.x$ (since $\varphi$ is equivariant there). Hence they agree on all of $X$, and thus $\pi$ is $G$-equivariant:
\[
\pi(h.z)=h.\pi(z)\quad\text{for all }h\in G,\; z\in X.
\]

Finally, $\pi$ is surjective. Indeed, for any $g\in G$, $\pi(g.x)=g.\pi(x)=g.y$, so $\pi(G.x)\supseteq G.y$. Since $(Y,G)$ is minimal, $\overline{G.y}=Y$. The image $\pi(X)$ is compact (hence closed) and contains the dense subset $G.y$, so $\pi(X)=Y$.

Thus $\pi$ is a factor map with $\pi(x)=y$.\par\smallskip
\textbf{Step 6.} Conclusion. We have shown:
- If $\pi\colon X\to Y$ is a factor map with $\pi(x)=y$, then for every neighborhood $V$ of $y$ there exists a neighborhood $U$ of $x$ such that $G(x,U)\subseteq G(y,V)$.
- Conversely, if for every neighborhood $V$ of $y$ there exists a neighborhood $U$ of $x$ with $G(x,U)\subseteq G(y,V)$, then there exists a continuous, equivariant, surjective map $\pi\colon X\to Y$ with $\pi(x)=y$.

This completes the proof of Lemma \ref{lem:factor}.\par\smallskip

\paragraph{How the three judges scored it.}
\textsc{Gemini-3} and \textsc{Opus-4.5} accept all seven steps and accept the proof;
\textsc{GPT-5} accepts six and rejects \emph{Step 3}, hence rejects the whole proof
(Table~\ref{tab:gpt5-disagreement}).

\begin{table}[h]
\centering\small
\begin{tabular}{lccc}
\toprule
& \textsc{Gemini-3} & \textsc{GPT-5} & \textsc{Opus-4.5} \\
\midrule
Step 0 (setup)                       & \cmark & \cmark & \cmark \\
Step 1 ($\Rightarrow$ direction)     & \cmark & \cmark & \cmark \\
Step 2 ($\Leftarrow$: stabilizers)   & \cmark & \cmark & \cmark \\
Step 3 (uniform continuity)          & \cmark & \xmark & \cmark \\
Step 4 (extension to $X$)            & \cmark & \cmark & \cmark \\
Step 5 (equivariance, surjectivity)  & \cmark & \cmark & \cmark \\
Step 6 (conclusion)                  & \cmark & \cmark & \cmark \\
\midrule
\textbf{Overall verdict} & \textcolor{accept}{\textbf{accept}} & \textcolor{reject}{\textbf{reject}} & \textcolor{accept}{\textbf{accept}} \\
\bottomrule
\end{tabular}
\caption{Per-step verdicts on the \emph{same} proof of Lemma~\ref{lem:factor}. The only
disagreement is Step~3, which \textsc{GPT-5} alone rejects.}
\label{tab:gpt5-disagreement}
\end{table}

\paragraph{Why the rejection is unwarranted.}
The highlighted clause $V\times V\subseteq\delta_Y$ need not hold for an arbitrary entourage
$\delta_Y$, so it is indeed unjustified. \textbf{But it is never used.} In Step~3, $V$ is
used only through the hypothesis (to obtain $U$ with $G(x,U)\subseteq G(y,V)$), and the
entourage $\delta_Y$ is applied \emph{directly} via equicontinuity at the end; deleting the
clause leaves Steps~3--5 unchanged. The proof is therefore correct, as \textsc{Gemini-3} and
\textsc{Opus-4.5} confirm.

\paragraph{The judge concedes the proof is correct, then rejects it.}
\textsc{GPT-5}'s verbatim rationale grants that the \emph{main argument is correct} and that
the offending clause is \emph{not used later}, yet converts this immaterial aside into a
rejection of the entire proof:

{\bfseries\small \textsc{GPT-5} \textemdash{} Step 3 \textemdash{} emitted verdict: \texttt{incorrect}}\par\vspace{3pt}
\small\itshape
\emph{The main argument establishing uniform continuity is correct}, but the claim that
$V := \delta_Y[y]$ satisfies $V\times V \subset \delta_Y$ is generally false and
unjustified. $\delta_Y[y]$ is a neighborhood of $y$, but need not have
$V\times V \subset \delta_Y$. \textbf{Although this inclusion is not used later}, the step
contains an incorrect assertion, so it must be marked incorrect.

\paragraph{Takeaway.}
This failure differs from a judge that simply miscalculates: the local observation is
correct---the parenthetical clause really is unjustified---but the step-level protocol has no
notion of \emph{materiality}. A single incorrect aside that contributes nothing to the
deduction fails the step, and one failed step fails the whole proof. Holistic judging, which
weighs whether the proof \emph{as a whole} establishes the claim, does not exhibit this
brittleness. Because \textsc{GPT-5} is both the strictest judge and the most frequent lone
rejecter, this bias disproportionately depresses the accept rates it reports; cross-judge
disagreement is common, with split overall verdicts on $20\text{--}35\%$ of proofs depending
on the prover.

\end{document}